\documentclass{article}

\usepackage{microtype}
\usepackage{graphicx}
\usepackage{subcaption}
\usepackage{booktabs} 

\usepackage{hyperref}

\newcommand{\algorithmicreturn}{\textbf{return}}
\makeatletter
\newcommand{\RETURN}{\ALC@it\algorithmicreturn\ }
\makeatother
\usepackage[preprint]{icml2026}

\usepackage{amsmath}
\usepackage{amssymb}
\usepackage{mathtools}
\usepackage{amsthm}

\usepackage[capitalize,noabbrev]{cleveref}

\theoremstyle{plain}
\newtheorem{theorem}{Theorem}[section]

\newtheorem{lemma}[theorem]{Lemma}
\newtheorem{corollary}[theorem]{Corollary}
\theoremstyle{definition}

\theoremstyle{remark}

\usepackage[textsize=tiny]{todonotes}

\usepackage{multirow}

\icmltitlerunning{P-EAGLE: Parallel-Drafting EAGLE with Scalable Training}

\begin{document}

\twocolumn[
  \icmltitle{P-EAGLE: Parallel-Drafting EAGLE with Scalable Training}



  \icmlsetsymbol{equal}{*}

  \begin{icmlauthorlist}
    \icmlauthor{Mude Hui}{equal,yyy,intern}
    \icmlauthor{Xin Huang}{equal,comp}
    \icmlauthor{Jaime Campos Salas}{comp}
    \icmlauthor{Yue Sun}{comp}
    \icmlauthor{Nathan Pemberton}{comp}
    \icmlauthor{Xiang Song}{comp}
    \icmlauthor{Ashish Khetan}{comp}
    \icmlauthor{George Karypis}{comp}
   
  \end{icmlauthorlist}

  \icmlaffiliation{yyy}{University of California, Santa Cruz}
  \icmlaffiliation{comp}{AWS Amazon}
  \icmlaffiliation{intern}{Work done during an internship at AWS.}

  \icmlcorrespondingauthor{Xin Huang}{xinxh@amazon.com}

  \icmlkeywords{Speculative Decoding, Large Language Models, Parallel Drafting, Inference Optimization}

  \vskip 0.3in
]



\printAffiliationsAndNotice{\icmlEqualContribution}

\begin{abstract}


Reasoning LLMs produce longer outputs, requiring speculative decoding drafters trained on extended sequences. Parallel drafting—predicting multiple tokens per forward pass—offers latency benefits over sequential generation, but training complexity scales quadratically with the product of sequence length and parallel positions, rendering long-context training impractical. We present P(arallel)-EAGLE, which transforms EAGLE from autoregressive to parallel multi-token prediction via a learnable shared hidden state. To scale training to long contexts, we develop a framework featuring attention mask pre-computation and sequence partitioning techniques, enabling gradient accumulation \textit{within} individual sequences for parallel-prediction training. We implement P-EAGLE in vLLM and demonstrate speedups of 1.10×–1.36× over autoregressive EAGLE-3 across GPT-OSS 120B, 20B, and Qwen3-Coder 30B.

\end{abstract}

\section{Introduction}

Autoregressive decoding in large language models (LLMs) presents a fundamental efficiency challenge: each token requires a complete forward pass through billions of parameters, rendering inference memory-bandwidth bound. Speculative decoding~\citep{chen2023accelerating} addresses this limitation by having a lightweight draft model propose multiple candidates autoregressively, which are then verified by the target model in a single forward pass. 

Among various approaches in speculative decoding, EAGLE~\citep{li2024eagle,li2025eagle3} has achieved widespread adoption in production inference systems, including vLLM~\citep{kwon2023efficient}, SGLang~\citep{zheng2024sglang}, and TensorRT-LLM~\citep{tensorrt_llm}, delivering 2--3$\times$ speedups over standard autoregressive decoding. EAGLE conditions token predictions on hidden states from the target model, leveraging contextual representations that standalone drafters must learn independently through multiple transformer layers. This enables a compact single-layer architecture comprising only 2--5\% of target model parameters. However, EAGLE generates draft tokens autoregressively: producing $K$ tokens requires $K$ sequential forward passes, creating significant drafting overhead.

Parallel drafting presents a promising approach to eliminate the overhead of autoregressive decoding.
Multiple prior works have explored parallel drafting strategies for speculative decoding \cite{gloeckle2024better,xiao2024parallelspec,cai2024medusa,monea2023pass,lin2025bita}.  ParallelSpec~\citep{xiao2024parallelspec} proposed parallel drafting with a single transformer layer, but omits critical implementation details---notably whether and how target model hidden states are utilized---and does not address the memory scaling challenges that arise from extended training sequences with multiple parallel prediction positions.  PARD~\citep{an2025pard} addresses this complexity through Conditional Drop-token (COD) training, which retains progressively fewer sequence positions at later parallel-prediction depths to reduce effective sequence length. However, both methods face scalability challenges when training on long sequences.

The scalability limitations of existing parallel drafting methods become consequential in modern inference workloads. Reasoning-capable models produce substantially longer outputs: for example, on UltraChat dataset~\citep{ding2023enhancing}, GPT-OSS 120B~\citep{openai2025gptoss120bgptoss20bmodel} exhibits median sequence lengths of 3,891 tokens, with P90 reaching 10,800 (Figure~\ref{fig:length_distribution}).
\begin{figure}[ht]
\centering
\includegraphics[width=1\columnwidth]{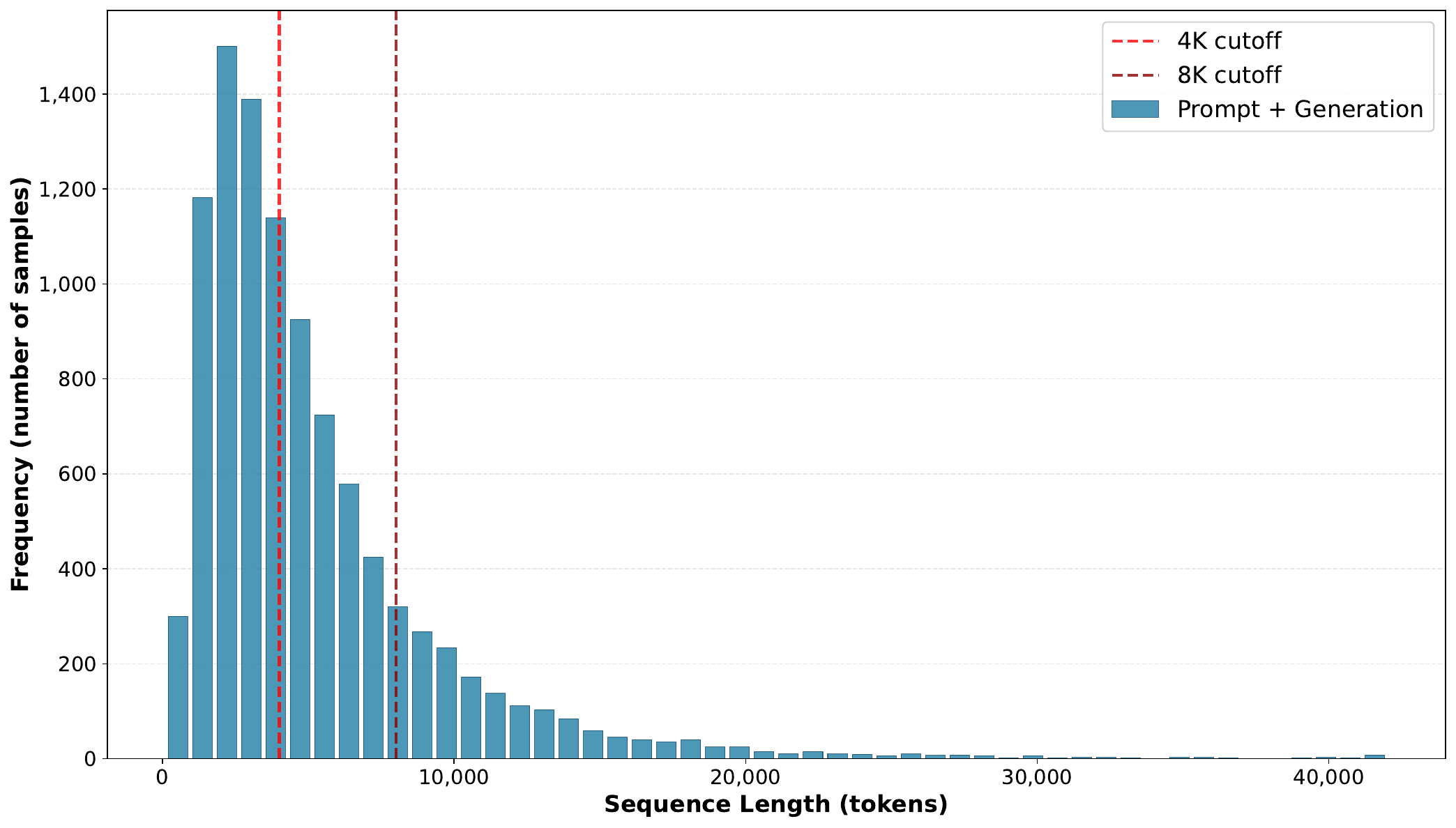}
\caption{Sequence length (prompt + generation) distribution on UltraChat dataset with GPT-OSS 120B. Reasoning level: Medium. Median: 3,891 tokens; P90: 10,800 tokens; P99: 20,000 tokens.}
\label{fig:length_distribution}
\end{figure}
Draft models trained on shorter sequences encounter a distribution mismatch when deployed on such workloads, exhibiting up to 25\% reduction in acceptance rate on the extended reasoning traces (Table~\ref{tab:method_comparison}). Effective parallel drafting therefore requires scalability to long training contexts that match inference distributions. Unlike standard training where gradient accumulation operates across examples, parallel multi-token prediction amplifies memory pressure—the effective sequence length grows linearly with the number of parallel prediction positions, posing optimization challenges absent from autoregressive training.


To quantify this scalability gap, we compare ParallelSpec~\citep{xiao2024parallelspec} and PARD~\citep{an2025pard} with our method under identical training conditions\footnote{ParallelSpec does not release code or sufficient training details; we implemented their method following the paper. PARD supports only standalone drafters; we adapted it to EAGLE's training framework.}. The results are shown in Table~\ref{tab:method_comparison}. ParallelSpec encounters extremely low acceptance length at 1K and 4K training contexts, and out-of-memory failures at 8K+ due to quadratic attention scaling. PARD's per-batch mask construction becomes computationally prohibitive beyond 4K contexts (see Section~\ref{sec:method_scalable}). Our method scales to 20K tokens while maintaining competitive acceptance length. For hyperparameters and hardware, see Appendix~\ref{appendix-sec:scal-compare}
\begin{table}[h]
\centering
\caption{Acceptance length (AL) comparison on the MT-Bench dataset. Target model: GPT-OSS 120B. Speculation length: $5$. ``Infeas.'' denotes computational infeasibility with 10+h per epoch.}
\label{tab:method_comparison}
\small
\scalebox{0.85}{
\begin{tabular}{lccccc}
\toprule
\multirow{2}{*}{Method} & \multirow{2}{*}{Layers} & \multicolumn{4}{c}{Training context length} \\
 & & 1K & 4K & 8K & 20K \\
\midrule
ParallelSpec + EAGLE 3 & 1  &  1.5 & 1.6 & OOM & OOM \\
PARD + EAGLE 3 & 4  & \textbf{2.4} & Infeas. & OOM & OOM \\
\textbf{Ours (P-EAGLE)} & 4 & \textbf{2.4 } & \textbf{2.8} & \textbf{2.9} & \textbf{3.0} \\
\bottomrule
\end{tabular}
}

\end{table}

We present \textbf{P(arallel-drafting) EAGLE}, which transforms EAGLE from autoregressive generation to parallel multi-token prediction with scalable training.
Our contributions are as follows.

\begin{enumerate}

\item \textbf{Scalable training framework for long contexts} (Section~\ref{sec:method_scalable}): We develop amortized mask construction and sequence partitioning to address a unique challenge in parallel-prediction training: attention memory scales quadratically with the product of sequence length and prediction depth. Our sequence partitioning technique splits a single sequence into segments for gradient accumulation while preserving attention dependencies.

\item \textbf{EAGLE-based parallel drafting architecture}
    (Section~\ref{sec:method_hidden}): We introduce a learnable shared
    hidden state that enables EAGLE to generate multiple draft tokens
    in a single forward pass. Theoretical analysis shows attention alone encodes sufficient positional information, eliminating the need for position-specific hidden states. Ablations demonstrate this simple design outperforms four position-aware alternatives by 7--15\%.

    \item \textbf{Optimzed training recipe} (Section~\ref{sec:method_recipe}): Through systematic ablations, we establish P-EAGLE training best practices including architecture depth, train-inference prediction depth alignment, and embedding strategies.

    \item \textbf{Production deployment} (Section~\ref{sec:experiments}): We implement P-EAGLE in vLLM. Through comprehensive benchmarking, P-EAGLE demonstrates consistent speedups of 1.10×–1.36× over autoregressive EAGLE-3 across GPT-OSS 120B, 20B~\citep{openai2025gptoss120bgptoss20bmodel} and Qwen3-Coder 30B~\citep{qwen3technicalreport}.

\end{enumerate}

\section{Architecture}
\label{sec:method_hidden}


\begin{figure}[t]
\centering
\includegraphics[width=0.98\columnwidth]{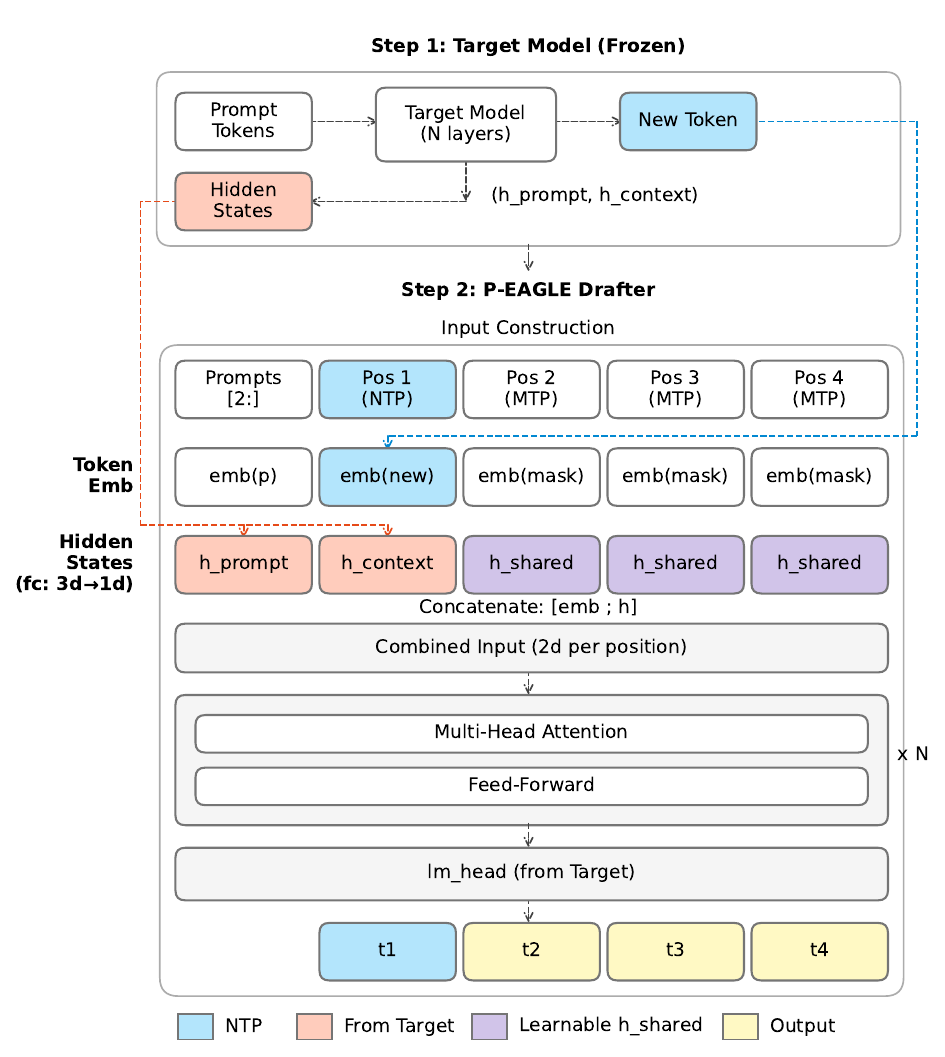}

\caption{P-EAGLE architecture. The target model (top) processes prompt tokens and produces hidden states from layer indexes 2, $L/2$, and $L-1$ (concatenated to $3d$ dimensions), where $L$ is the number of decoder layers. The P-EAGLE drafter (bottom) takes these hidden states for the Next-Token Prediction (NTP) position (Pos 1), which operates like standard autoregressive prediction with actual context. Multi-Token Prediction (MTP) positions (Pos 2-4) use a learnable shared hidden state $h_{\text{shared}}$ since they lack preceding hidden states. Token embeddings are combined with projected hidden states and processed through $N$ transformer layers.}

\label{fig:architecture}
\end{figure}

We present the P-EAGLE architecture in Figure~\ref{fig:architecture}. The drafter follows the LLaMA 3 architecture with rotary positional embeddings (RoPE)~\citep{su2024roformer}.

\textbf{Background.} We first overview autoregressive EAGLE using Figure~\ref{fig:architecture} as reference. To predict token $t_1$, the hidden state from the target model is concatenated with the token embedding, processed through transformer layers to produce a hidden vector, and passed through the LM head. This corresponds to the Next-Token Prediction (NTP) position (Pos 1) in the figure. To generate token $t_2$, the drafter takes the predicted token $t_1$ and the hidden vector used to predict $t_1$ (before passing through the LM head) as input for the next forward pass. Similarly, generating $t_3$ uses the predicted token $t_2$ and the hidden vector used to predict $t_2$. Producing $K$ draft tokens thus requires $K$ sequential forward passes.

\textbf{Challenge for parallel drafting.} Generating $K$ tokens in parallel eliminates sequential forward passes but introduces a new problem: positions predicting $t_2, t_3, \ldots$ (which we call MTP positions) lack the predicted tokens and hidden vectors from previous steps. P-EAGLE addresses this with two learnable parameters: a shared hidden state $h_{\text{shared}}$ that substitutes for the missing hidden vectors, and a mask token embedding that substitutes for the unknown previous tokens. This enables all $K$ tokens to be generated in a single forward pass. We compare four alternative hidden state designs in Section~\ref{sec:hidden_ablation}, finding this simple shared approach outperforms position-aware variants by 7--15\%.

\textbf{Additional design choices.} P-EAGLE unfreezes the token embeddings inherited from the target model, as the mask token embedding must be learned to encode meaningful input for MTP positions (Section~\ref{sec:recipe_embedding}). We also use a deeper architecture, with four layers achieving 46\% higher acceptance length than one layer (Section~\ref{sec:recipe_layers}).

\section{Scalable Training Framework for Long Contexts}
\label{sec:method_scalable}


Training a parallel token prediction model requires extending each sequence of length $n$ to accommodate $K$ parallel prediction depths, where depth $k$ predicts the token $k+1$ positions ahead. Without optimization, this creates $n \times K$ total positions with $O((nK)^2)$ attention complexity, causing out-of-memory failures at long sequences.

PARD~\citep{an2025pard} addresses this with Conditional Drop-token (COD) sampling, which reduces the number of positions at each prediction depth (also referred to as ``groups'' in PARD). Specifically, COD applies geometric decay: depth 0 retains all $n$ positions, depth 1 randomly retains $n \times r$ positions, depth 2 retains $n \times r^2$, and so on, where $r \in (0, 1)$ is the retention rate. The total positions across all depths becomes $n \times (1 + r + r^2 + \cdots + r^{K-1})$ rather than $n \times K$, significantly reducing attention memory. However, because COD samples different positions randomly for each training example, PARD must construct a custom attention mask per example. This mask enforces causal constraints across prediction depths: positions at depth $d$ can only attend to positions from earlier depths, not to depths $d+1$ or beyond (which do not exist at inference time). Constructing these masks requires $O((nK)^2)$ operations per example, becoming prohibitively expensive when training on long sequences (Table~\ref{tab:method_comparison}).

We address this bottleneck with two techniques: \textbf{amortized mask construction} (Section~\ref{sec:precomputed-mask}) and \textbf{sequence partitioning} (Section~\ref{sec:dependency-splitting}).

\subsection{Amortized Mask Construction}
\label{sec:precomputed-mask}

The key insight enabling our approach is that the causal structure across prediction depths is \textit{position-invariant}: the attention pattern for positions $0$ through $n$ is identical regardless of total sequence length. This means a mask for any sequence can be obtained by extracting the top-left $(n \times K) \times (n \times K)$ submatrix from a pre-computed maximum-length mask, as illustrated in Figure~\ref{fig:attention_mask_crop}.

\begin{figure}[t]
\centering
\includegraphics[width=0.48\textwidth]{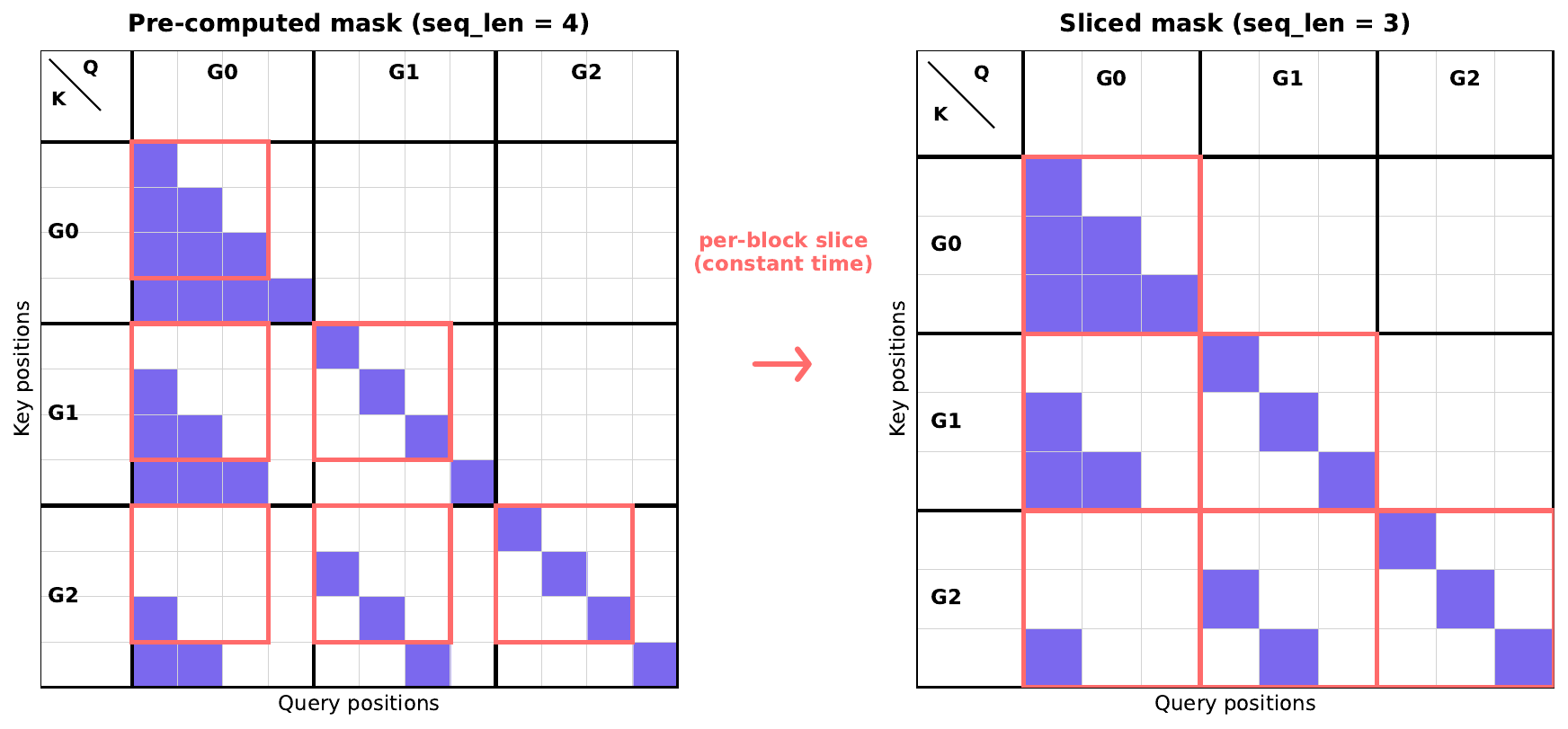}
\caption{Position-invariance of causal attention across prediction depths. G0, G1, G2 in the figure denote prediction depths 0, 1, 2, where depth $d$ predicts the token $d+1$ positions ahead. The mask for a shorter sequence (right) is exactly the top-left submatrix of a longer sequence's mask (left), enabling constant-time retrieval.}
\label{fig:attention_mask_crop}
\end{figure}

We exploit this property by constructing the attention mask once at training initialization for the maximum sequence length. During training, per-example masks are obtained via tensor slicing---a constant-time view operation requiring no additional memory allocation. The one-time initialization cost is amortized across millions of training steps, with a fixed memory footprint independent of dataset size.

The practical impact is substantial. Table~\ref{tab:training_overhead} shows that at 2048-token sequences, PARD's per-example mask construction causes 48$\times$ data loading slowdown and 5$\times$ epoch time increase. Our pre-computed approach eliminates this bottleneck entirely.

\begin{table}[h]
\centering
\caption{Training overhead (2048 tokens, $K=8$). Data loading measured on 128 examples. Epoch measured on UltraChat (200K examples), 8$\times$H200 GPUs.}
\label{tab:training_overhead}
\small
\begin{tabular}{lcc}
\toprule
Method & Load (128 ex.) & Epoch \\
\midrule
EAGLE-3\footnotemark & 14.8s & 2.5h \\
PARD & 718.5s (48$\times$) & 12+h (5$\times$) \\
\textbf{Ours} & 17.5s & 1.8h \\
\bottomrule
\end{tabular}
\end{table}
\footnotetext{EAGLE-3 uses Training-Time Test~\citep{li2025eagle3}, requiring multiple forward passes per example.}

\subsection{Sequence Partitioning}
\label{sec:dependency-splitting}
Pre-computed masks eliminate construction overhead, but memory remains a bottleneck as sequences grow. Consider an 8192-token sequence with $K=8$ prediction depths and retention rate $r=0.8$. The total positions across all depths follows $n \times (1 - r^K)/(1 - r)$, yielding approximately 34K positions. Attention memory scales as $O(L^2)$ with total positions $L$, while embeddings and output logits scale as $O(L)$. Training at longer sequences requires managing this memory growth, which introduces two challenges.
 
\textbf{Challenge 1: Within-sequence gradient accumulation.} Standard gradient accumulation addresses memory constraints by splitting a batch into micro-batches, where each micro-batch contains one or more complete training examples. This assumes individual examples fit in memory. When a single sequence exceeds memory, a new approach is needed. We propose partitioning the sequence itself into segments, processing each with a separate forward-backward pass, and accumulating gradients across segments. To our knowledge, this within-sequence gradient accumulation is unique to parallel-prediction training and has not been explored in prior work.

\textbf{Challenge 2: Preserving cross-depth dependencies.} Sequence partitioning is further complicated by COD's attention structure. The causal constraint requires that position $p$ at depth $d$ attends to position $p-1$ at depth $d-1$. Meanwhile, COD's random sampling creates different position sets at each depth, as illustrated in Figure~\ref{fig:splitting}. When partitioning by position index, a position and its dependency may land in different segments, violating the attention pattern.

Figure~\ref{fig:splitting} provides a concrete example. Consider $n=16$ tokens with $K=4$ prediction depths (G0--G3 in the figure) and retention rate $r=0.7$. Depth 0 contains all 16 positions, while depths 1--3 contain progressively fewer due to COD sampling---suppose depth 1 retains positions $\{1,3,4,6,7,9,10,12,14,15\}$, depth 2 retains $\{2,5,7,8,11,13,15\}$, and depth 3 retains $\{3,6,9,12,14\}$, yielding 38 total positions. To partition into 2 segments: if we assign by depth-0 indices (positions 0--7 to Segment 0, positions 8--15 to Segment 1), position 8 at depth 2 lands in Segment 1, but its dependency---position 7 at depth 1---lands in Segment 0. This breaks the required attention pattern.

  \begin{figure}[t]
  \centering
  \includegraphics[width=\columnwidth]{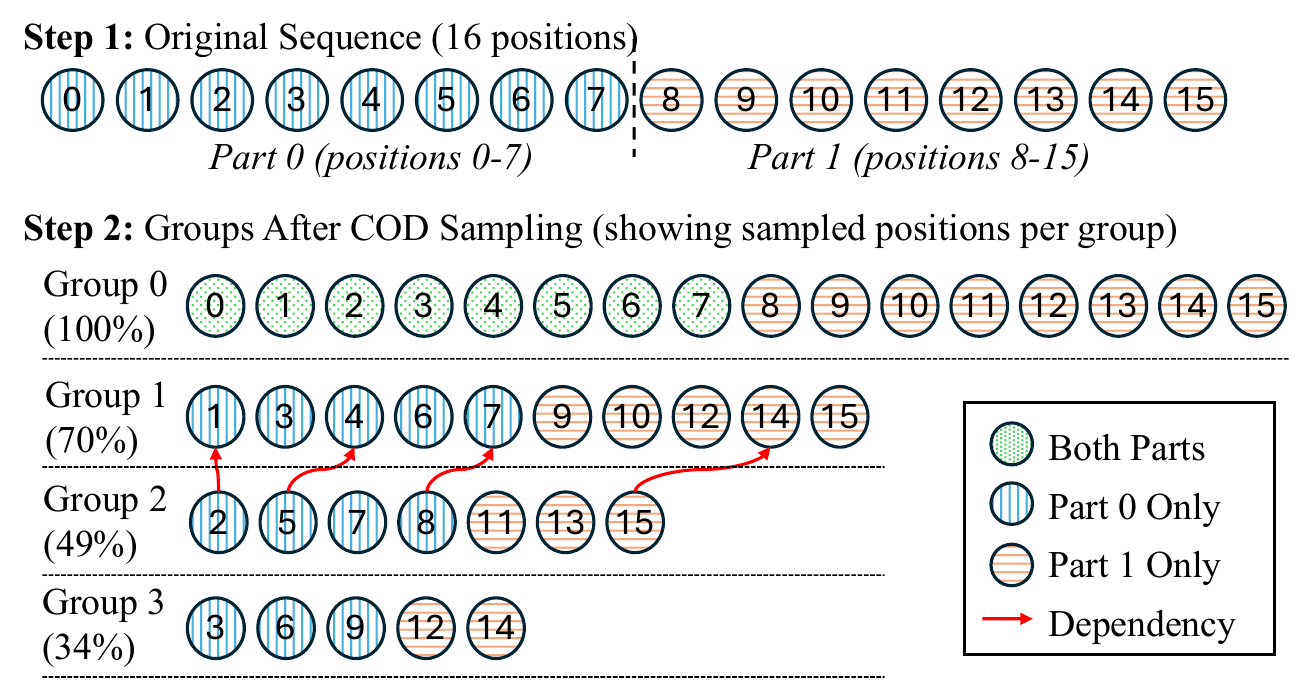}
  \caption{Sequence partitioning for within-sequence gradient accumulation. Example: $n=16$ tokens, $K=4$ prediction depths (G0--G3 denote groups in PARD terminology, where group $g$ predicts the token $g+1$ positions ahead). Depth 0 contains all positions; depths 1--3 contain progressively fewer due to COD sampling. Partitioning by depth-0 indices causes dependency violations: position 8 at depth 2 depends on position 7 at depth 1, but they land in different segments. Our algorithm tracks assignments iteratively across depths to preserve dependencies.}
  \label{fig:splitting}
  \end{figure}

  \textbf{Our solution: sequence partitioning technique.}  We track segment assignments iteratively across prediction depths. For depths 0 and 1, positions are assigned to segments based on their index. For depths $d \geq 2$, each position is assigned to the same segment as its dependency at depth $d-1$. This iterative propagation guarantees that position $p$ and its dependency $p-1$ always reside in the same segment. Additionally, each segment includes depth-0 positions cumulatively up to its boundary to satisfy causal attention. Algorithm~\ref{alg:splitting} presents the pseudocode for the sequence partitioning technique.

\textbf{Result.} With $S$ segments, peak attention memory reduces from $O(L^2)$ to $O(L^2/S^2)$, enabling within-sequence gradient accumulation while preserving all cross-depth attention dependencies.

\begin{algorithm}[t]
\caption{Sequence Partitioning }
\label{alg:splitting}
\begin{algorithmic}[1]
\REQUIRE $\mathcal{P} = \{\mathcal{P}_0, \mathcal{P}_1, \ldots, \mathcal{P}_{K-1}\}$: sampled positions for $K$ depths
\REQUIRE $S$: number of segments for gradient accumulation
\REQUIRE $L$: sequence length
\ENSURE $\mathcal{A}$: segment assignment for all positions
\ENSURE $\mathcal{N}$: cumulative depth-0 positions per segment

\STATE \textcolor{gray}{\textit{// Initialize segment boundaries}}
\STATE $\mathcal{B} \leftarrow \{0, \frac{L}{S}, \frac{2L}{S}, \ldots, L\}$

\STATE \textcolor{gray}{\textit{// Phase 1: Assign segments for depths 0 and 1 by position}}
\FOR{$g \in \{0, 1\}$}
    \FOR{$p \in \mathcal{P}_g$}
        \STATE $\mathcal{A}_g[p] \leftarrow \max\{s : \mathcal{B}_s \leq p\}$
    \ENDFOR
\ENDFOR

\STATE \textcolor{gray}{\textit{// Phase 2: Propagate assignments via dependencies}}
\FOR{$g = 2$ \textbf{to} $K-1$}
    \FOR{$p \in \mathcal{P}_g$}
        \STATE $\mathcal{A}_g[p] \leftarrow \mathcal{A}_{g-1}[p-1]$ \hfill $\triangleright$ \textit{inherit from dependent position}
    \ENDFOR
\ENDFOR

\STATE \textcolor{gray}{\textit{// Phase 3: Accumulate NTP positions for causal attention}}
\FOR{$s = 0$ \textbf{to} $S-1$}
    \STATE $\mathcal{N}_s \leftarrow \{p \in \mathcal{P}_0 : p < \mathcal{B}_{s+1}\}$
\ENDFOR

\RETURN $\mathcal{A}, \mathcal{N}$
\end{algorithmic}
\end{algorithm}



\section{Training Recipe}
\label{sec:method_recipe}



This section presents ablation studies validating P-EAGLE's design choices. We evaluate P-EAGLE on two out-of-distribution benchmarks: HumanEval~\citep{chen2021codex} and MT-Bench~\citep{zheng2023judging}, reporting acceptance length at speculation depth $K$=5. Unless otherwise noted, ablations use LLaMA 3.1 8B~\citep{grattafiori2024llama} as the target model and a single decoder layer for P-EAGLE.


  \subsection{Hidden State Design}
  \label{sec:hidden_ablation}

A natural question is whether MTP positions should have position-specific hidden states rather than sharing one. We conduct this study using GPT-OSS 20B as the target model with a 4-layer P-EAGLE drafter. We evaluated four augmentation strategies---adding depth-specific encodings (to distinguish MTP position 2, 3, 4...), injecting projected NTP hidden states, or combining both. All underperformed the simple shared hidden state by 7--15\%. Results are shown in Table~\ref{tab:hidden_strategies}.

  The ``+ NTP hidden'' variants inject the preceding NTP position's hidden state into MTP positions. The regularized variant adds learnable scaling: $h_{\text{MTP}} = h_{\text{shared}} + \alpha \cdot \text{proj}(h_{\text{NTP}})$, where $\alpha$ controls context injection strength. Formulations for all variants are in Appendix~\ref{appendix:hidden_ablation_details}.

  \begin{table}[h]
  \centering
  \caption{Hidden state ablation on HumanEval. Target: GPT-OSS 20B. Training: 4-layer P-EAGLE, 20 epochs on OpenCodeInstruct. Evaluation: speculation length 5.}
  \label{tab:hidden_strategies}
  \scalebox{0.9}{
  \begin{tabular}{lcc}
  \toprule
  Strategy &  Acc.\ Length & $\Delta$\% \\
  \midrule
  \textbf{Baseline (learnable shared)} &  \textbf{3.16} & --- \\
  + depth-specific encoding & 2.85 & $-$9.8\% \\
  + NTP hidden + depth encoding  & 2.68 & $-$15.2\% \\
  + NTP hidden only & 2.81 & $-$11.1\% \\
  + regularized NTP hidden  & 2.94 & $-$7.0\% \\
  \bottomrule
  \end{tabular}
  }
  \end{table}

\textbf{Theoretical justification.} We attribute this to functional redundancy. Rotary position embeddings already encode absolute position, from which parallel-prediction depth is computable, eliminating the need for explicit depth encodings. Similarly, the attention mechanism allows MTP positions to access NTP context directly, making auxiliary context injection superfluous. We formalize this in Appendix~\ref{appendix:theory}: absolute position is uniquely recoverable from RoPE-based attention scores (Theorem~\ref{thm:score_injectivity}).

\textbf{Empirical confirmation.} The regularized variant provides direct evidence that the model actively learns to minimize context injection. The learnable $\alpha$ decays exponentially from 0.1 to 0.029 over training---a 71\% decrease---converging toward zero. Moreover, the baseline (no context injection) outperforms the regularized variant throughout training: at epoch 20, baseline achieves 57.9\% MTP accuracy compared to 54.6\% for the regularized variant. This confirms that context injection hurts performance, and the model mitigates this by driving $\alpha$ toward zero. See Figure~\ref{fig:alpha_trajectory} in Appendix~\ref{appendix:hidden_ablation_details} for detailed trajectories.

\subsection{Increasing Model Capacity}
\label{sec:recipe_layers}

The most impactful factor is model depth. Autoregressive EAGLE achieves strong acceptance with a single layer because each position conditions on the previously-generated token and hidden states, while P-EAGLE generates all tokens in parallel without access to intermediate tokens.

\begin{table}[h]
\centering
\caption{Effect of decoder layer count on P-EAGLE acceptance length. $\Delta$\% reports the relative change w.r.t.\ the 1-layer baseline on each benchmark.}
\label{tab:layer_ablation}
\small
\begin{tabular}{lccc}
\toprule
Layers & HumanEval & MT-Bench & $\Delta$\% \\
\midrule
1 & 2.69 & 2.41 & — \\
2 & 3.58 & 2.76 & +33.1\% / +14.5\% \\
4 & 3.92 & 3.04 & +45.7\% / +26.1\% \\
\bottomrule
\end{tabular}
\end{table}

Table~\ref{tab:layer_ablation} shows that increasing from 1 to 2 layers provides the largest gain (+33\% on HumanEval), with 4 layers achieving an additional +9.5\%. Although additional layers increase per-forward-pass latency, Section~\ref{sec:experiments} demonstrates that the latency savings from parallel drafting (i.e., reduction from $K$ sequential forward passes to one forward pass) far outweigh this overhead.

\subsection{Unfreezing the Embedding Layer}
\label{sec:recipe_embedding}

Standard autoregressive EAGLE freezes the target model's embedding layer during training---the embeddings are already well-suited for representing actual tokens. P-EAGLE introduces a \textit{mask token} (i.e., a pre-defined unused token ID) for MTP positions. A frozen embedding layer cannot adapt to encode meaningful information for this mask token.

\begin{table}[h]
\centering
\caption{Effect of unfreezing the embedding layer on P-EAGLE acceptance length. $\Delta$\% reports the relative change w.r.t.\ the frozen-embedding baseline on each benchmark.}

\label{tab:embedding_ablation}
\small
\begin{tabular}{lccc}
\toprule
Freeze Emb. & HumanEval & MT-Bench & $\Delta$\% \\
\midrule
Yes (frozen) & 2.56 & 2.29 & — \\
No (trainable) & 2.69 & 2.41 & +5.1\% / +5.2\% \\
\bottomrule
\end{tabular}
\end{table}

Table~\ref{tab:embedding_ablation} confirms that unfreezing embeddings provides consistent +5\% improvement across both evaluation datasets. We hypothesize that the learned mask token embedding encodes a ``default next-token prior'' that serves as a meaningful starting point for parallel positions, complementing the learnable shared hidden state.

\subsection{Training vs. Inference Speculation Depth}
\label{sec:recipe_tasks}

P-EAGLE trains with $K_{\text{train}}$ parallel prediction groups and speculates with depth $K_{\text{infer}}$ at inference. A natural question is whether these should match.

\begin{table}[h]
\centering
\caption{Effect of training speculation depth on P-EAGLE acceptance length. $\Delta$\% reports the relative change w.r.t.\ the $K_{\text{tr}}{=}5,\,K_{\text{inf}}{=}5$ baseline on each benchmark.}

\label{tab:tasks_ablation}
\small
\begin{tabular}{ccccc}
\toprule
$K_{\text{tr}}$ & $K_{\text{inf}}$ & HumanEval & MT-Bench & $\Delta$\% \\
\midrule
5 & 5 & 2.41 & 2.20 & — \\
8 & 5 & 2.51 & 2.26 & +4.1\% / +2.7\% \\
\bottomrule
\end{tabular}
\end{table}

Table~\ref{tab:tasks_ablation} shows that training with $K_{\text{train}}=8$ while inferring at $K_{\text{infer}}=5$ yields +4\% improvement over matched settings. Training with longer prediction horizons encourages the model to learn more robust multi-step dependencies, with positions 6--8 providing additional supervision that benefits positions 1--5. Interestingly, PARD~\citep{an2025pard} demonstrates that standalone draft models can extrapolate in the opposite direction ($K_{\text{infer}} > K_{\text{train}}$).

\subsection{Extended Training Duration}
\label{sec:recipe_duration}

P-EAGLE faces a harder learning problem than autoregressive EAGLE: it must learn to extract position-specific information via attention rather than receiving it directly as input. We find that extended training duration improves acceptance, as shown in Table~\ref{tab:duration_ablation}. 

\begin{table}[h]
\centering
\caption{Effect of training duration on P-EAGLE acceptance length. $\Delta$\% reports the relative change w.r.t.\ the 20-epoch baseline on each benchmark.}

\label{tab:duration_ablation}
\small
\begin{tabular}{lccc}
\toprule
Epochs & HumanEval & MT-Bench & $\Delta$\% \\
\midrule
20 & 3.92 & 3.04 & — \\
40 & 3.98 & 3.15 & +1.5\% / +3.6\% \\
60 & 4.00 & 3.18 & +2.0\% / +4.6\% \\
\bottomrule
\end{tabular}
\end{table}

\subsection{Longer Training Sequences}
\label{sec:recipe_seqlen}

Finally, we find that increasing training sequence length improves acceptance. Table~\ref{tab:seqlen_ablation} shows that increasing sequence length from 512 to 2048 yields +2\% improvement. Since LLaMA 3.1 8B is not a reasoning model, the gains are modest, especially on these short-context evaluation datasets. For reasoning-capable models with much longer outputs, as demonstrated in Table~\ref{tab:method_comparison}, the benefits of long-context training are more pronounced.

\begin{table}[h]
\centering
\caption{Effect of maximum training sequence length on P-EAGLE acceptance length. $\Delta$\% reports the relative change w.r.t.\ the 512-token baseline on each benchmark.}
\label{tab:seqlen_ablation}
\small
\begin{tabular}{lccc}
\toprule
Max Seq Len & HumanEval & MT-Bench & $\Delta$\% \\
\midrule
512 & 2.51 & 2.26 & — \\
2048 & 2.56 & 2.29 & +2.0\% / +1.3\% \\
\bottomrule
\end{tabular}
\end{table}




The key insight is that P-EAGLE requires ``capacity compensation'' for the information loss inherent in parallel prediction: more layers to transform shared input into position-specific representations, trainable embeddings to encode the mask token meaningfully, and extended training to master the harder attention-based context extraction. With these adaptations, P-EAGLE achieves acceptance rates competitive with autoregressive EAGLE while reducing drafting latency through parallelism. For comprehensive comparison, see Section~\ref{sec:experiments}.

\section{Experiments}
\label{sec:experiments}

We evaluate P-EAGLE across three popular open-source models, comparing against autoregressive (AR) EAGLE-3 on acceptance length and end-to-end inference throughput from high-performance inference framework vLLM.

\subsection{Experimental Setup}\label{sec:exp_setup}

We evaluate P-EAGLE on GPT-OSS 120B, GPT-OSS 20B, and Qwen3-Coder 30B. Based on the capacity analysis in Section~\ref{sec:recipe_layers}, we train P-EAGLE with 4 decoder layers. For comparison with 2-layer P-EAGLE, see Appendix~\ref{appendix:2layer}.

\textbf{Training configuration.} All models are trained with maximum sequence length 8192 tokens, parallel prediction groups $K_{\text{train}} = 8$, and COD down-sampling ratio 0.8. We use batch size 8 with micro-batch size 1 (8-step gradient accumulation), linear learning rate schedule with peak $1 \times 10^{-4}$ and warmup ratio 0.0025. Training is conducted on 8$\times$ H200 GPUs. We use identical training configurations across all three target models to ensure fair comparison.

\textbf{Training data.} We train both P-EAGLE and AR EAGLE 3 on three datasets: UltraChat~\citep{ding2023enhancing}, GSM-8K (train split)~\citep{cobbe2021gsm8k}, OpenCodeInstruct~\citep{ahmad2025opencodeinstruct}.

\textbf{Evaluation.} We evaluate on three out-of-distribution benchmarks: HumanEval (code generation)~\citep{chen2021codex}, MT-Bench (multi-turn conversation)~\citep{zheng2023judging}, and GSM-8K test split (mathematical reasoning)~\citep{cobbe2021gsm8k}.

\textbf{Baseline.} We compare P-EAGLE against AR EAGLE-3 as our primary baseline.  The baseline follows the canonical single-layer design from~\citet{li2025eagle3}\footnote{Under sequential generation, drafting latency scales with layer count, offsetting the marginal acceptance-length gains. Single-layer is thus optimal for AR EAGLE throughput.}. We also implemented the Harmonized Context Alignment (HCA) loss~\citep{zhang2024learning} (AR-specific) in training the baseline drafter, yielding strong absolute performance (e.g., acceptance length of 4.36 out of speculation length 5 on HumanEval for Qwen3-Coder 30B, approaching the theoretical maximum of 6.0). This represents a more challenging baseline than alternatives such as Medusa~\citep{cai2024medusa} or Lookahead decoding~\citep{fu2024break}.

ParallelSpec~\citep{xiao2024parallelspec} and PARD~\citep{an2025pard} are excluded from comparison due to scalability limitations demonstrated in Table~\ref{tab:method_comparison}. Both ParallelSpec and PARD encounter out-of-memory failures at 8K context length during training. Since production training for reasoning-capable models requires 8K+ context support, these methods cannot be evaluated under equivalent conditions.

\subsection{Acceptance Length Comparison}
\label{sec:exp_acceptance}

Before examining end-to-end performance, we establish that P-EAGLE achieves \textit{comparable acceptance length} to autoregressive EAGLE-3. This is a necessary condition: if parallel drafting substantially degraded draft quality, latency savings would be offset by reduced acceptance rates.

\textbf{Key insight.} The goal of this comparison is \textit{not} to demonstrate that P-EAGLE achieves higher acceptance length---rather, we show that parallel drafting can match autoregressive quality with modest additional capacity (2--4 layers vs.\ 1 layer). The true benefit of P-EAGLE lies in end-to-end throughput improvement (Section~\ref{sec:exp_e2e}), not acceptance length superiority.

\begin{table}[t]
\centering
\caption{Acceptance length comparison across three models on three out-of-distribution benchmarks. Speculation depth $K=5$, max new tokens 2048. Percentages denote improvement over AR EAGLE-3.}
\label{tab:main_results}
\resizebox{\columnwidth}{!}{
\begin{tabular}{llcc}
\toprule
\textbf{Model} & \textbf{Dataset} & \textbf{AR EAGLE-3} & \textbf{P-EAGLE (4L)} \\
\midrule
\multirow{4}{*}{GPT-OSS 120B}
& HumanEval & 3.5 & \textbf{3.5} (0.0\%) \\
& MT-Bench  & 2.7 & \textbf{2.9} (+10.2\%) \\
& GSM-8K    & 3.3 & \textbf{3.5} (+5.2\%) \\
& \textit{Average} & \textit{3.1} & \textit{\textbf{3.3} (+4.5\%)} \\
\midrule
\multirow{4}{*}{GPT-OSS 20B}
& HumanEval & 3.7 & \textbf{3.8} (+2.4\%) \\
& MT-Bench  & 3.4 & \textbf{3.4} (+1.5\%) \\
& GSM-8K    & 3.9 & \textbf{4.0} (+3.1\%) \\
& \textit{Average} & \textit{3.7} & \textit{\textbf{3.7} (+2.5\%)} \\
\midrule
\multirow{4}{*}{Qwen3-Coder 30B}
& HumanEval & 4.4 & \textbf{4.5} (+3.7\%) \\
& MT-Bench  & 3.0 & \textbf{3.0} (+0.3\%) \\
& GSM-8K    & 3.1 & \textbf{3.2} (+1.0\%) \\
& \textit{Average} & \textit{3.5} & \textit{\textbf{3.6} (+2.0\%)} \\
\bottomrule
\end{tabular}
}
\vspace{-2.em}
\end{table}


Table~\ref{tab:main_results} presents acceptance length results across all configurations. Key observations: (1) P-EAGLE matches or exceeds AR EAGLE-3 across all 9 model-dataset combinations. The average improvement is $+4.5\%$ for GPT-OSS 120B, $+2.5\%$ for GPT-OSS 20B, and $+2.0\%$ for Qwen3-Coder 30B. (2) Qwen3-Coder 30B presents the most challenging baseline: AR EAGLE-3 achieves acceptance length of 4.36 on HumanEval, approaching the theoretical maximum of 6.0 for $K=5$. Nevertheless, P-EAGLE consistently matches or exceeds the baseline across all three benchmarks (+3.7\% on HumanEval, +0.3\% on MT-Bench, +1.0\% on GSM-8K), confirming that parallel drafting maintains quality even when the autoregressive baseline is already highly optimized. We provide a comparison between 2-layer and 4-layer P-EAGLE in Appendix~\ref{appendix:2layer}.

  \begin{table*}
  \centering
 \caption{Output Tokens Per Second (OTPS) across speculation depths $K$ and concurrency levels (C). At each $K$, OTPS measures total throughput across all concurrent requests. \underline{Underline} indicates AR baseline (optimal $K$). P-EAGLE speedups (in parentheses) are relative to this baseline; \textbf{bold} indicates best speedup. HE=HumanEval, MT=MT-Bench, GSM=GSM-8K. All experiments use chain drafting and are measured on 1 H200 GPU.}
  \label{tab:otps}
  \scalebox{0.9}{
  \begin{tabular}{llc|ccc|ccc}
  \toprule
  & & & \multicolumn{3}{c|}{\textbf{C=2}} & \multicolumn{3}{c}{\textbf{C=4}} \\
  \textbf{Model} & \textbf{Method} & $K$ & HE & MT & GSM & HE & MT & GSM \\
  \midrule
  \multirow{6}{*}{120B}
  & \multirow{3}{*}{AR} & 3 & \underline{592} & \underline{516} & \underline{612} & \underline{864} & \underline{774} & 930 \\
  &  & 5 & 566 & 457 & 606 & 841 & 698 & \underline{936} \\
  &  & 7 & 503 & 402 & 555 & 774 & 620 & 869 \\
  \cmidrule{2-9}
  & \multirow{3}{*}{P-EAGLE} & 3 & 616 (1.04$\times$) & 537 (\textbf{1.04}$\times$) & 636 (1.04$\times$) & 883 (1.02$\times$) & 790 (1.02$\times$) & 939 (1.00$\times$) \\
  &  & 5 & 632 (\textbf{1.07}$\times$) & 536 (1.04$\times$) & 676 (\textbf{1.10}$\times$) & 898 (\textbf{1.04}$\times$) & 799 (\textbf{1.03}$\times$) & 993 (\textbf{1.06}$\times$) \\
  &  & 7 & 606 (1.02$\times$) & 516 (1.00$\times$) & 672 (1.10$\times$) & 879 (1.02$\times$) & 758 (0.98$\times$) & 988 (1.06$\times$) \\
  \midrule
  \multirow{6}{*}{20B}
  & \multirow{3}{*}{AR} & 3 & \underline{929} & \underline{894} & 957 & 1538 & \underline{1524} & 1612 \\
  &  & 5 & 912 & 851 & \underline{968} & \underline{1562} & 1486 & \underline{1696} \\
  &  & 7 & 833 & 776 & 912 & 1407 & 1387 & 1600 \\
  \cmidrule{2-9}
  & \multirow{3}{*}{P-EAGLE} & 3 & 1001 (1.08$\times$) & 965 (1.08$\times$) & 1064 (1.10$\times$) & 1597 (1.02$\times$) & 1663 (1.09$\times$) & 1762 (1.04$\times$) \\
  &  & 5 & 1173 (1.26$\times$) & 1087 (1.22$\times$) & 1218 (1.26$\times$) & 1876 (1.20$\times$) & 1883 (\textbf{1.24}$\times$) & 2054 (1.21$\times$) \\
  &  & 7 & 1205 (\textbf{1.30}$\times$) & 1135 (\textbf{1.27}$\times$) & 1320 (\textbf{1.36}$\times$) & 1950 (\textbf{1.25}$\times$) & 1882 (1.23$\times$) & 2147 (\textbf{1.27}$\times$) \\
  \midrule
  \multirow{6}{*}{Qwen 30B}
  & \multirow{3}{*}{AR} & 3 & 674 & \underline{547} & \underline{604} & 1046 & \underline{869} & 932 \\
  &  & 5 & \underline{737} & 521 & 602 & \underline{1160} & 825 & \underline{992} \\
  &  & 7 & 725 & 493 & 576 & 1107 & 784 & 932 \\
  \cmidrule{2-9}
  & \multirow{3}{*}{P-EAGLE} & 3 & 696 (0.94$\times$) & 559 (1.02$\times$) & 613 (1.01$\times$) & 1066 (0.92$\times$) & 871 (1.00$\times$) & 967 (0.97$\times$) \\
  &  & 5 & 820 (1.11$\times$) & 571 (\textbf{1.04}$\times$) & 653 (\textbf{1.08}$\times$) & 1225 (1.06$\times$) & 907 (\textbf{1.04}$\times$) & 1049 (\textbf{1.06}$\times$) \\
  &  & 7 & 860 (\textbf{1.17}$\times$) & 558 (1.02$\times$) & 638 (1.06$\times$) & 1285 (\textbf{1.11}$\times$) & 882 (1.01$\times$) & 1030 (1.04$\times$) \\
  \bottomrule
  \end{tabular}
  }
  \vspace{-1.em}
  \end{table*}

\subsection{Experimental results from vLLM}
\label{sec:exp_e2e}


We implement P-EAGLE's parallel drafting in vLLM and report the Output Tokens Per Second (OTPS) across all three models and benchmarks in Table~\ref{tab:otps}. A key advantage of P-EAGLE is its ability to draft more tokens without proportional cost increase—while AR EAGLE-3 achieves optimal throughput at K=3, P-EAGLE efficiently scales to K=5–7 by generating all draft tokens in a single forward pass. This higher speculation depth reduces the total number of generation iterations required to complete a response. As a result, P-EAGLE achieves up to 1.36× speedup at concurrency C=2, with gains varying by model size: 1.27×–1.36× for 20B, 1.04×–1.10× for 120B, and 1.04×–1.17× for Qwen 30B. At higher concurrency (C=4), speedups remain strong for 20B (1.24×–1.27×) while moderating for larger models (1.03×–1.11×) as verification latency becomes dominant—for MoE models, expert routing overhead scales with batch size, shifting the bottleneck from drafting to verification. For Qwen 30B on HumanEval at K=3, P-EAGLE shows slight slowdowns (0.98× at C=2, 0.92× at C=4) due to its deeper architecture: a single 4-layer forward pass can exceed three sequential 1-layer passes at low speculation depth. This overhead is amortized at K=5–7, where speedups reach 1.11×–1.17×.

\section{Related Work}
A variety of techniques have been explored to accelerate inference in large language models (LLMs), such as quantization and knowledge distillation. However, these approaches generally trade model performance for speed. Speculative sampling enables lossless acceleration by verifying drafted tokens in parallel using the target model~\citep{leviathan2023fast, chen2023accelerating}. Early works primarily relied on smaller standalone models~\citep{miao2023specinfer} or retrieval-based approaches, such as REST~\citep{he2023rest}, to generate draft candidates. More recent methods like EAGLE, Mixture of Attentions, and HASS exploit intermediate representations of the target model to improve decoding efficiency~\cite{li2025eagle3,zimmer2025mixtureattentions,zhang2024learning}. GLIDE and CAPE~\cite{glide_cape} instead reuse the target model’s key--value cache and confidence scores. However, these techniques still only generate one token per forward-pass through the draft model, limiting throughput.


Parallel token drafting mitigates the autoregressive bottleneck by predicting multiple tokens per forward pass. Self-drafting methods modify the target model to produce multiple speculative tokens in parallel~\cite{gloeckle2024better,cai2024medusa,monea2023pass}. ParallelSpec shows that a discrete, EAGLE-style parallel drafter can outperform autoregressive EAGLE and self-drafting Medusa in speed~\cite{xiao2024parallelspec}, while PARD reduces the training cost of parallel prediction via Conditional Drop-token~\cite{an2025pard}. In parallel, Falcon improves semi-autoregressive speculative decoding through dependency-aware training and a custom decoding tree, targeting higher draft quality under SAR generation~\cite{gao2025falcon}. Cascade Speculative Drafting instead reduces drafting latency by cascading multiple draft models (down to a statistical model) and allocating computation by draft position~\cite{chen2024cascade}. In contrast, P-EAGLE uses a target-conditioned EAGLE drafter with parallel multi-token prediction, and focuses on scalable long-context training for extended reasoning workloads.

As generation lengths have increased, there have been several approaches to supporting long-context workloads. Some techniques limit the set of tokens preserved in the KV-cache to limit attention overheads~\cite{xiao2024efficient,Beltagy2020Longformer,zhangH2O,sun2024triforce}. Others quantize the KV cache~\cite{xiao2023smoothquant,hooper2024kvquant,shengFlexGen,liu2024kivi,Yue2024WKVQuantQW}. While these techniques increase the maximum supportable context length, they can lead to loss of model performance. Sadhukhan et al. combined these lossy techniques with speculative decoding to provide net lossless speedup~\cite{chen2024magicdec}. These techniques are orthogonal to our approach and could be combined to further increase context lengths.

\section{Conclusion}

We presented P-EAGLE, transforming EAGLE-style speculative decoding from autoregressive to parallel multi-token prediction via learnable shared hidden state and mask token embeddings. Our training framework with mask pre-computation and sequence partitioning enables scalable long-context training, addressing a key bottleneck in prior parallel drafting approaches. We implement P-EAGLE in vLLM, achieving 1.10×–1.36× speedup over autoregressive EAGLE-3 across GPT-OSS 120B, GPT-OSS 20B, and Qwen3-Coder 30B. These results establish parallel drafting as a viable direction for production LLM acceleration. 
\newpage




\nocite{langley00}

\bibliography{references}

@article{zhang2024learning,
  title={Learning harmonized representations for speculative sampling},
  author={Zhang, Lefan and Wang, Xiaodan and Huang, Yanhua and Xu, Ruiwen},
  journal={arXiv preprint arXiv:2408.15766},
  year={2024}
}

@article{li2024eagle,
  title={EAGLE: Speculative Sampling Requires Rethinking Feature Uncertainty},
  author={Li, Yuhui and Wei, Fangyun and Zhang, Chao and Zhang, Hongyang},
  journal={arXiv preprint arXiv:2401.15077},
  year={2024}
}

@inproceedings{gao2025falcon,
  title={Falcon: Faster and parallel inference of large language models through enhanced semi-autoregressive drafting and custom-designed decoding tree},
  author={Gao, Xiangxiang and Xie, Weisheng and Xiang, Yiwei and Ji, Feng},
  booktitle={Proceedings of the AAAI Conference on Artificial Intelligence},
  volume={39},
  number={22},
  pages={23933--23941},
  year={2025}
}

@article{chen2024cascade,
  title={Cascade speculative drafting for even faster llm inference},
  author={Chen, Ziyi and Yang, Xiaocong and Lin, Jiacheng and Sun, Chenkai and Chang, Kevin C and Huang, Jie},
  journal={Advances in Neural Information Processing Systems},
  volume={37},
  pages={86226--86242},
  year={2024}
}

@article{an2025pard,
  title={{PARD}: Accelerating {LLM} Inference with Low-Cost Parallel Draft Model Adaptation},
  author={An, Zihao and Bai, Huajun and Liu, Ziqiong and Li, Dong and Barsoum, Emad},
  journal={arXiv preprint arXiv:2504.18583},
  year={2025}
}

@article{li2025eagle3,
  title={{EAGLE-3}: Scaling up Inference Acceleration of Large Language Models via Training-Time Test},
  author={Li, Yuhui and Wei, Fangyun and Zhang, Chao and Zhang, Hongyang},
  journal={arXiv preprint arXiv:2503.01840},
  year={2025}
}

@article{leviathan2023fast,
  title={Fast Inference from Transformers via Speculative Decoding},
  author={Leviathan, Yaniv and Kalman, Matan and Matias, Yossi},
  journal={International Conference on Machine Learning},
  year={2023}
}

@article{chen2023accelerating,
  title={Accelerating Large Language Model Decoding with Speculative Sampling},
  author={Chen, Charlie and Borgeaud, Sebastian and Irving, Geoffrey and Lespiau, Jean-Baptiste and Sifre, Laurent and Jumper, John},
  journal={arXiv preprint arXiv:2302.01318},
  year={2023}
}

@article{miao2023specinfer,
  title={SpecInfer: Accelerating Generative Large Language Model Serving with Tree-based Speculative Inference and Verification},
  author={Miao, Xupeng and Oliaro, Gabriele and Zhang, Zhihao and Cheng, Xinhao and Wang, Zeyu and Wong, Rae Ying Yee and Chen, Zhuoming and Arfeen, Daiyaan and Abhyankar, Reyna and Jia, Zhihao},
  journal={arXiv preprint arXiv:2305.09781},
  year={2023}
}

@article{cai2024medusa,
  title={Medusa: Simple LLM Inference Acceleration Framework with Multiple Decoding Heads},
  author={Cai, Tianle and Li, Yuhong and Geng, Zhengyang and Peng, Hongwu and Lee, Jason D and Chen, Deming and Dao, Tri},
  journal={arXiv preprint arXiv:2401.10774},
  year={2024}
}

@article{he2023rest,
  title={REST: Retrieval-Based Speculative Decoding},
  author={He, Zhenyu and Zhong, Zexuan and Cai, Tianle and Lee, Jason D and He, Di},
  journal={arXiv preprint arXiv:2311.08252},
  year={2023}
}

@article{fu2024break,
  title={Break the Sequential Dependency of LLM Inference Using Lookahead Decoding},
  author={Fu, Yichao and Bailis, Peter and Stoica, Ion and Zhang, Hao},
  journal={arXiv preprint arXiv:2402.02057},
  year={2024}
}

@article{su2024roformer,
  title={{RoFormer}: Enhanced Transformer with Rotary Position Embedding},
  author={Su, Jianlin and Lu, Yu and Pan, Shengfeng and Murtadha, Ahmed and Wen, Bo and Liu, Yunfeng},
  journal={Neurocomputing},
  volume={568},
  pages={127063},
  year={2024},
  publisher={Elsevier}
}

@article{liu2025rethinking,
  title={Rethinking RoPE: A Mathematical Blueprint for N-dimensional Positional Encoding},
  author={Liu, Haiping and Zhou, Hongpeng},
  journal={arXiv preprint arXiv:2504.06308},
  year={2025}
}

@article{xiao2024parallelspec,
  title={Parallelspec: Parallel drafter for efficient speculative decoding},
  author={Xiao, Zilin and Zhang, Hongming and Ge, Tao and Ouyang, Siru and Ordonez, Vicente and Yu, Dong},
  journal={arXiv preprint arXiv:2410.05589},
  year={2024}
}

@inproceedings{kwon2023efficient,
  title={Efficient Memory Management for Large Language Model Serving with PagedAttention},
  author={Woosuk Kwon and Zhuohan Li and Siyuan Zhuang and Ying Sheng and Lianmin Zheng and Cody Hao Yu and Joseph E. Gonzalez and Hao Zhang and Ion Stoica},
  booktitle={Proceedings of the ACM SIGOPS 29th Symposium on Operating Systems Principles},
  year={2023}
}

@article{zheng2024sglang,
  title={Sglang: Efficient execution of structured language model programs},
  author={Zheng, Lianmin and Yin, Liangsheng and Xie, Zhiqiang and Sun, Chuyue Livia and Huang, Jeff and Yu, Cody Hao and Cao, Shiyi and Kozyrakis, Christos and Stoica, Ion and Gonzalez, Joseph E and others},
  journal={Advances in neural information processing systems},
  volume={37},
  pages={62557--62583},
  year={2024}
}

@misc{tensorrt_llm,
  author       = {NVIDIA},
  title        = {TensorRT-LLM: High-Performance Inference Optimization for Large Language Models},
  year         = {2023--2025},
  howpublished = {\url{https://github.com/NVIDIA/TensorRT-LLM}},
  note         = {Version X.Y.Z},
}

@inproceedings{ding2023enhancing,
  title={Enhancing chat language models by scaling high-quality instructional conversations},
  author={Ding, Ning and Chen, Yulin and Xu, Bokai and Qin, Yujia and Hu, Shengding and Liu, Zhiyuan and Sun, Maosong and Zhou, Bowen},
  booktitle={Proceedings of the 2023 Conference on Empirical Methods in Natural Language Processing},
  pages={3029--3051},
  year={2023}
}

@misc{openai2025gptoss120bgptoss20bmodel,
      title={gpt-oss-120b \& gpt-oss-20b Model Card}, 
      author={OpenAI},
      year={2025},
      eprint={2508.10925},
      archivePrefix={arXiv},
      primaryClass={cs.CL},
      url={https://arxiv.org/abs/2508.10925}, 
}

@misc{qwen3technicalreport,
      title={Qwen3 Technical Report}, 
      author={Qwen Team},
      year={2025},
      eprint={2505.09388},
      archivePrefix={arXiv},
      primaryClass={cs.CL},
      url={https://arxiv.org/abs/2505.09388},
}

@article{grattafiori2024llama,
  title={The llama 3 herd of models},
  author={Grattafiori, Aaron and Dubey, Abhimanyu and Jauhri, Abhinav and Pandey, Abhinav and Kadian, Abhishek and Al-Dahle, Ahmad and Letman, Aiesha and Mathur, Akhil and Schelten, Alan and Vaughan, Alex and others},
  journal={arXiv preprint arXiv:2407.21783},
  year={2024}
}

@article{chen2021codex,
  title={Evaluating Large Language Models Trained on Code},
  author={Mark Chen and Jerry Tworek and Heewoo Jun and Qiming Yuan and Henrique Ponde de Oliveira Pinto and Jared Kaplan and Harri Edwards and Yuri Burda and Nicholas Joseph and Greg Brockman and Alex Ray and Raul Puri and Gretchen Krueger and Michael Petrov and Heidy Khlaaf and Girish Sastry and Pamela Mishkin and Brooke Chan and Scott Gray and Nick Ryder and Mikhail Pavlov and Alethea Power and Lukasz Kaiser and Mohammad Bavarian and Clemens Winter and Philippe Tillet and Felipe Petroski Such and Dave Cummings and Matthias Plappert and Fotios Chantzis and Elizabeth Barnes and Ariel Herbert-Voss and William Hebgen Guss and Alex Nichol and Alex Paino and Nikolas Tezak and Jie Tang and Igor Babuschkin and Suchir Balaji and Shantanu Jain and William Saunders and Christopher Hesse and Andrew N. Carr and Jan Leike and Josh Achiam and Vedant Misra and Evan Morikawa and Alec Radford and Matthew Knight and Miles Brundage and Mira Murati and Katie Mayer and Peter Welinder and Bob McGrew and Dario Amodei and Sam McCandlish and Ilya Sutskever and Wojciech Zaremba},
  year={2021},
  eprint={2107.03374},
  archivePrefix={arXiv},
  primaryClass={cs.LG}
}

@article{zheng2023judging,
  title={Judging llm-as-a-judge with mt-bench and chatbot arena},
  author={Zheng, Lianmin and Chiang, Wei-Lin and Sheng, Ying and Zhuang, Siyuan and Wu, Zhanghao and Zhuang, Yonghao and Lin, Zi and Li, Zhuohan and Li, Dacheng and Xing, Eric and others},
  journal={Advances in neural information processing systems},
  volume={36},
  pages={46595--46623},
  year={2023}
}

@article{cobbe2021gsm8k,
  title={Training Verifiers to Solve Math Word Problems},
  author={Cobbe, Karl and Kosaraju, Vineet and Bavarian, Mohammad and Chen, Mark and Jun, Heewoo and Kaiser, Lukasz and Plappert, Matthias and Tworek, Jerry and Hilton, Jacob and Nakano, Reiichiro and Hesse, Christopher and Schulman, John},
  journal={arXiv preprint arXiv:2110.14168},
  year={2021}
}

@article{ahmad2025opencodeinstruct,
      title={OpenCodeInstruct: A Large-scale Instruction Tuning Dataset for Code LLMs}, 
      author={Wasi Uddin Ahmad and Aleksander Ficek and Mehrzad Samadi and Jocelyn Huang and Vahid Noroozi and Somshubra Majumdar and Boris Ginsburg},
      year={2025},
      eprint={2504.04030},
      archivePrefix={arXiv},
      primaryClass={cs.CL},
      url={https://arxiv.org/abs/2504.04030}, 
}

@misc{monea2023pass,
      title={PaSS: Parallel Speculative Sampling}, 
      author={Giovanni Monea and Armand Joulin and Edouard Grave},
      year={2023},
      eprint={2311.13581},
      archivePrefix={arXiv},
      primaryClass={cs.CL},
      url={https://arxiv.org/abs/2311.13581}, 
}

@article{lin2025bita,
  title={BiTA: Bi-Directional Tuning for Lossless Acceleration in Large Language Models},
  author={Lin, Feng and Yi, Hanling and Yang, Yifan and Li, Hongbin and Yu, Xiaotian and Lu, Guangming and Xiao, Rong},
  journal={Expert Systems with Applications},
  pages={127305},
  year={2025},
  publisher={Elsevier}
}

@inproceedings{
    xiao2024efficient,
    title={Efficient Streaming Language Models with Attention Sinks},
    author={Guangxuan Xiao and Yuandong Tian and Beidi Chen and Song Han and Mike Lewis},
    booktitle={The Twelfth International Conference on Learning Representations},
    year={2024},
    url={https://openreview.net/forum?id=NG7sS51zVF}
}

@article{Beltagy2020Longformer,
  title={Longformer: The Long-Document Transformer},
  author={Iz Beltagy and Matthew E. Peters and Arman Cohan},
  journal={arXiv:2004.05150},
  year={2020},
}

@inproceedings{zhangH2O,
    author = {Zhang, Zhenyu and Sheng, Ying and Zhou, Tianyi and Chen, Tianlong and Zheng, Lianmin and Cai, Ruisi and Song, Zhao and Tian, Yuandong and R\'{e}, Christopher and Barrett, Clark and Wang, Zhangyang and Chen, Beidi},
    title = {H2O: heavy-hitter oracle for efficient generative inference of large language models},
    year = {2023},
    publisher = {Curran Associates Inc.},
    address = {Red Hook, NY, USA},
    booktitle = {Proceedings of the 37th International Conference on Neural Information Processing Systems},
    articleno = {1506},
    numpages = {50},
    location = {New Orleans, LA, USA},
    series = {NIPS '23}
}

@inproceedings{sun2024triforce,
    title={TriForce: Lossless Acceleration of Long Sequence Generation with Hierarchical Speculative Decoding},
    author={Hanshi Sun and Zhuoming Chen and Xinyu Yang and Yuandong Tian and Beidi Chen},
    booktitle={First Conference on Language Modeling},
    year={2024},
    url={https://openreview.net/forum?id=HVK6nl3i97}
}

@article{chen2024magicdec,
  title={MagicDec: Breaking the Latency-Throughput Tradeoff for Long Context Generation with Speculative Decoding},
  author={Chen, Jian and Tiwari, Vashisth and Sadhukhan, Ranajoy and Chen, Zhuoming and Shi, Jinyuan and Yen, Ian En-Hsu and Chen, Beidi},
  journal={arXiv preprint arXiv:2408.11049},
  year={2024}
}

@inproceedings{gloeckle2024better,
    author = {Gloeckle, Fabian and Idrissi, Badr Youbi and Rozi\`{e}re, Baptiste and Lopez-Paz, David and Synnaeve, Gabriel},
    title = {Better \& faster large language models via multi-token prediction},
    year = {2024},
    publisher = {JMLR.org},
    booktitle = {Proceedings of the 41st International Conference on Machine Learning},
    articleno = {629},
    numpages = {29},
    location = {Vienna, Austria},
    series = {ICML'24}
}

@InProceedings{xiao2023smoothquant,
    title = {{S}mooth{Q}uant: Accurate and Efficient Post-Training Quantization for Large Language Models},
    author = {Xiao, Guangxuan and Lin, Ji and Seznec, Mickael and Wu, Hao and Demouth, Julien and Han, Song},
    booktitle = {Proceedings of the 40th International Conference on Machine Learning},
    year = {2023}
}

@article{hooper2024kvquant,
  title={KVQuant: Towards 10 Million Context Length LLM Inference with KV Cache Quantization},
  author={Hooper, Coleman and Kim, Sehoon and Mohammadzadeh, Hiva and Mahoney, Michael W and Shao, Yakun Sophia and Keutzer, Kurt and Gholami, Amir},
  journal={arXiv preprint arXiv:2401.18079},
  year={2024}
}

@inproceedings{shengFlexGen,
    author = {Sheng, Ying and Zheng, Lianmin and Yuan, Binhang and Li, Zhuohan and Ryabinin, Max and Chen, Beidi and Liang, Percy and R\'{e}, Christopher and Stoica, Ion and Zhang, Ce},
    title = {FlexGen: high-throughput generative inference of large language models with a single GPU},
    year = {2023},
    publisher = {JMLR.org},
    booktitle = {Proceedings of the 40th International Conference on Machine Learning},
    articleno = {1288},
    numpages = {23},
    location = {Honolulu, Hawaii, USA},
    series = {ICML'23}
}

@article{liu2024kivi,
  title={KIVI: A Tuning-Free Asymmetric 2bit Quantization for KV Cache},
  author={Liu, Zirui and Yuan, Jiayi and Jin, Hongye and Zhong, Shaochen and Xu, Zhaozhuo and Braverman, Vladimir and Chen, Beidi and Hu, Xia},
  journal={arXiv preprint arXiv:2402.02750},
  year={2024}
}

@article{Yue2024WKVQuantQW,
  title={WKVQuant: Quantizing Weight and Key/Value Cache for Large Language Models Gains More},
  author={Yuxuan Yue and Zhihang Yuan and Haojie Duanmu and Sifan Zhou and Jianlong Wu and Liqiang Nie},
  journal={ArXiv},
  year={2024},
  volume={abs/2402.12065},
  url={https://api.semanticscholar.org/CorpusID:267750952}
}

@misc{zimmer2025mixtureattentions,
      title={Mixture of Attentions For Speculative Decoding}, 
      author={Matthieu Zimmer and Milan Gritta and Gerasimos Lampouras and Haitham Bou Ammar and Jun Wang},
      year={2025},
      eprint={2410.03804},
      archivePrefix={arXiv},
      primaryClass={cs.CL},
      url={https://arxiv.org/abs/2410.03804}, 
}

@inproceedings{glide_cape,
    author = {Du, Cunxiao and Jiang, Jing and Yuanchen, Xu and Wu, Jiawei and Yu, Sicheng and Li, Yongqi and Li, Shenggui and Xu, Kai and Nie, Liqiang and Tu, Zhaopeng and You, Yang},
    title = {GLIDE with a CAPE: a low-hassle method to accelerate speculative decoding},
    year = {2024},
    publisher = {JMLR.org},
    booktitle = {Proceedings of the 41st International Conference on Machine Learning},
    articleno = {465},
    numpages = {17},
    location = {Vienna, Austria},
    series = {ICML'24}
}
\bibliographystyle{icml2026}

\newpage
\appendix
\onecolumn


\section{Training configuration of ParallelSpec and PARD}
\label{appendix-sec:scal-compare}
All models are trained on 8 H200 GPUs with a global batch size of 64 and 8-step gradient accumulation. The target model is GPT-OSS 120B. We use a linear learning-rate schedule with a peak learning rate of $1 \times 10^{-4}$ and a warmup ratio of 0.0025.

\section{Theoretical Justification for Redundancy of Hidden State Augmentation}
\label{appendix:theory}

We provide theoretical justification for why augmenting the shared hidden state---whether through prediction depth embeddings or projected NTP context---is redundant. The key observation is that prediction depth $g$ is a deterministic function of sequence position $p$: in the MTP training structure, each position maps to exactly one prediction depth via $g(p)$. We prove that absolute position $p$ can be uniquely recovered from RoPE-based attention scores. Since $g(p)$ is computable from $p$, introducing depth-dependent hidden states provides no additional information to the model.

\citet{liu2025rethinking} prove that the mapping $\delta \mapsto R_\delta$ from relative position between a pair of query and key $(\mathbf{q}, \mathbf{k}) \in \mathbb{R}^d \times \mathbb{R}^d$ to RoPE rotation matrix is injective, i.e., distinct relative positions always produce distinct rotation matrices. With a reference token at known position (e.g., BOS at position 0), this implies absolute position is recoverable from the rotation matrix. However, transformers do not observe rotation matrices directly; they compute scalar attention scores $\mathbf{q}^\top R_\delta \mathbf{k}$. The question then arises: is the attention score mapping $\delta \mapsto \mathbf{q}^\top R_\delta \mathbf{k}$ also injective? We prove that the set of query-key pairs for which injectivity fails has Lebesgue measure zero.

\begin{theorem}[\citet{liu2025rethinking}]
\label{thm:liu_matrix_injectivity}
The mapping $\delta \mapsto R_\delta$ from relative position to RoPE rotation matrix is injective, i.e., $R_{\delta_1} = R_{\delta_2} \implies \delta_1 = \delta_2$.
\end{theorem}

\begin{lemma}
\label{lem:measure_zero_kernel}
Let $M \in \mathbb{R}^{d \times d}$ be a non-zero matrix. The set $\mathcal{Z}_M = \{(\mathbf{q}, \mathbf{k}) \in \mathbb{R}^d \times \mathbb{R}^d : \mathbf{q}^\top M \mathbf{k} = 0\}$ has Lebesgue measure zero in $\mathbb{R}^{2d}$.
\end{lemma}

\begin{proof}
Since $M \neq 0$, there exist indices $(i^*, j^*)$ such that $M_{i^* j^*} \neq 0$. The bilinear form $\mathbf{q}^\top M \mathbf{k} = \sum_{i,j} q_i M_{ij} k_j$ is a polynomial in $2d$ variables with the non-zero monomial $M_{i^* j^*} q_{i^*} k_{j^*}$. Since a non-zero polynomial defines a proper algebraic hypersurface, its zero set has Lebesgue measure zero.
\end{proof}

\begin{theorem}[Attention Score-Level Injectivity]
\label{thm:score_injectivity}
The set of $(\mathbf{q}, \mathbf{k}) \in \mathbb{R}^d \times \mathbb{R}^d$ for which the attention score function $f_{\mathbf{q},\mathbf{k}}(\delta) = \mathbf{q}^\top R_\delta \mathbf{k}$ fails to be injective in relative position $\delta$ has Lebesgue measure zero.
\end{theorem}

\begin{proof}
Let $\mathbf{q}, \mathbf{k} \in \mathbb{R}^d$ be query and key vectors, and consider any $\delta_1 \neq \delta_2$. By Theorem~\ref{thm:liu_matrix_injectivity}, $R_{\delta_1} \neq R_{\delta_2}$, thus $M = R_{\delta_1} - R_{\delta_2} \neq 0$. The condition $f_{\mathbf{q},\mathbf{k}}(\delta_1) = f_{\mathbf{q},\mathbf{k}}(\delta_2)$ can be rewritten as $\mathbf{q}^\top R_{\delta_1} \mathbf{k} = \mathbf{q}^\top R_{\delta_2} \mathbf{k}$, or equivalently $\mathbf{q}^\top (R_{\delta_1} - R_{\delta_2}) \mathbf{k} = \mathbf{q}^\top M \mathbf{k} = 0$. By Lemma~\ref{lem:measure_zero_kernel}, since $M \neq 0$, the set of $(\mathbf{q}, \mathbf{k})$ satisfying $\mathbf{q}^\top M \mathbf{k} = 0$ has Lebesgue measure zero.
\end{proof}

\begin{corollary}[Absolute Position Recovery]
\label{cor:absolute_position}
Let $r$ be a reference position with key vector $\mathbf{k}_r \in \mathbb{R}^d$, and let $m$ be a query position with query vector $\mathbf{q} \in \mathbb{R}^d$. The set of query vectors for which the attention score $s_{m,r} = \mathbf{q}^\top R_{r-m} \mathbf{k}_r$ fails to uniquely determine $m$ has Lebesgue measure zero.
\end{corollary}

\begin{proof}
Since $r$ is fixed, distinct absolute positions $m_1 \neq m_2$ correspond to distinct relative positions $\delta_1 = r - m_1 \neq r - m_2 = \delta_2$. By Theorem~\ref{thm:score_injectivity}, the set of $(\mathbf{q}, \mathbf{k}_r)$ for which the attention score function $\delta \mapsto \mathbf{q}^\top R_\delta \mathbf{k}_r$ fails to be injective has Lebesgue measure zero. Therefore, $s_{m_1,r} = \mathbf{q}^\top R_{\delta_1} \mathbf{k}_r \neq \mathbf{q}^\top R_{\delta_2} \mathbf{k}_r = s_{m_2,r}$ except on a Lebesgue measure-zero set, and absolute position $m = r - \delta$ is uniquely recoverable.
\end{proof}

\subsection{Application to Hidden State Augmentation}

Consider the prediction depth embedding approach, which adds a learnable embedding to differentiate prediction depth: $h_{\text{MTP}}^{(g)} = h_{\text{shared}} + e_{\text{depth}}^{(g)}$, where $g \in \{1, \ldots, K\}$ denotes the prediction depth (1-step ahead, 2-steps ahead, etc.). The prediction depth is a deterministic function of sequence position: in the MTP training structure, each position $p$ maps to exactly one depth $g(p)$. By Corollary~\ref{cor:absolute_position}, absolute position is already recoverable from RoPE attention scores. Since $g(p)$ is computable from $p$, the explicit embedding $e_{\text{depth}}^{(g)}$ provides no information beyond what is already encoded in RoPE attention patterns.

For approaches that inject projected NTP context into hidden states, analogous redundancy arises: MTP positions already attend to NTP positions and receive their information through the attention output, making the additional projection pathway redundant. The general principle is that in RoPE-based transformers, explicit injection of information already accessible through attention is representationally unnecessary and harmful to optimization.

\subsection{Hidden State Ablation Details}
\label{appendix:hidden_ablation_details}

We evaluated five hidden state strategies for MTP positions. In PARD~\citep{an2025pard}, prediction depth is referred to as ``group index'' where group $g \in \{1, \ldots, K\}$ predicts the $(g+1)$-th future token. We use the more descriptive term ``prediction depth'' throughout.

\begin{itemize}
    \item \textbf{Baseline (learnable shared)}: A single learnable vector $h_{\text{shared}}$ shared across all MTP positions. This is our recommended approach.

    \item \textbf{+ depth-specific encoding}: Add a learnable embedding to differentiate prediction depth: $h_{\text{MTP}}^{(g)} = h_{\text{shared}} + e_{\text{depth}}^{(g)}$, where $g \in \{1, \ldots, K\}$ indicates predicting 1-step ahead, 2-steps ahead, etc. This is distinct from sequence position IDs (RoPE) used in attention.

    \item \textbf{+ NTP hidden + depth encoding}: Project the last NTP hidden state and combine with prediction depth embedding: $h_{\text{MTP}}^{(g)} = h_{\text{shared}} + \text{proj}(h_{\text{ntp}}) + e_{\text{depth}}^{(g)}$, where $\text{proj}(\cdot)$ is a linear layer. This provides both explicit context from the NTP position and prediction depth information.

    \item \textbf{+ NTP hidden only}: Project the last NTP hidden state without depth embedding: $h_{\text{MTP}} = h_{\text{shared}} + \text{proj}(h_{\text{ntp}})$, injecting context while relying on attention for depth differentiation.

    \item \textbf{+ regularized NTP hidden}: Add projected NTP hidden state with dropout regularization and learnable scaling: $h_{\text{MTP}} = h_{\text{shared}} + \alpha \cdot \text{dropout}(\text{proj}(h_{\text{ntp}}))$. The scalar $\alpha$ is initialized to 0.1, allowing the model to start near baseline behavior ($\alpha \approx 0$ ignores context) and learn whether to increase $\alpha$ if context helps. Dropout (rate 0.1) prevents overfitting to specific context patterns.
\end{itemize}

All alternatives underperformed the simple learnable shared hidden state (baseline) by 7--15\% (see Table~\ref{tab:hidden_strategies} in main paper). The theoretical analysis above explains this result: the augmentation provides redundant information that interferes with optimization.

  \begin{figure}[h]
  \centering
  \includegraphics[width=0.9\textwidth]{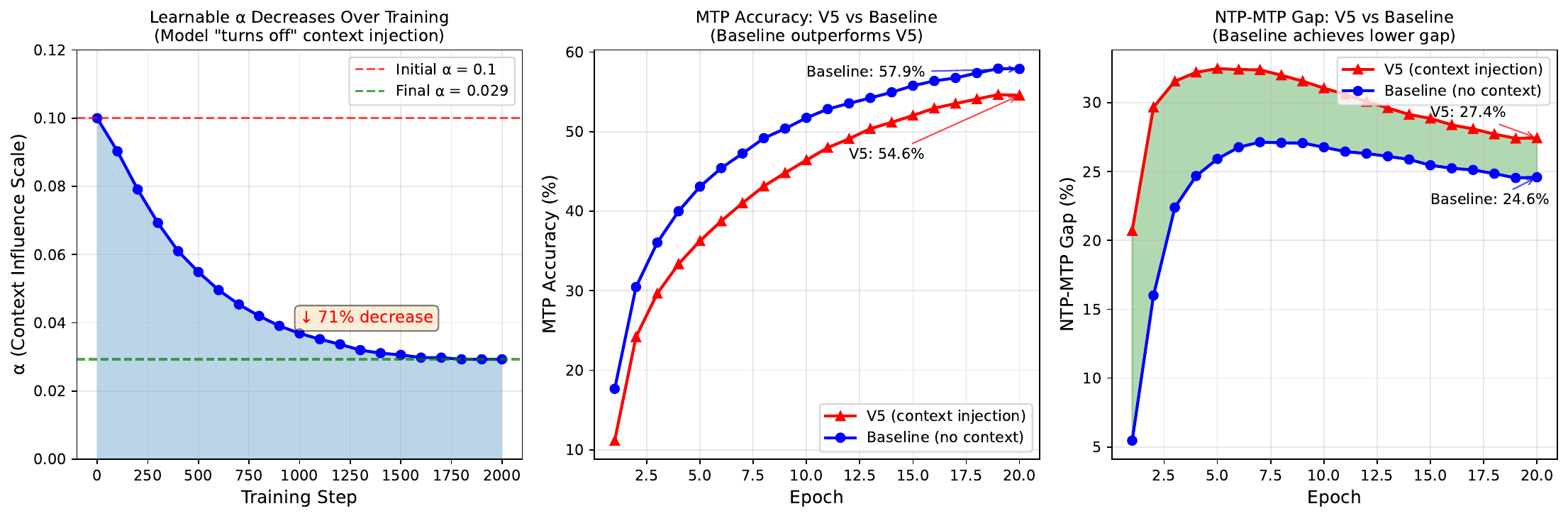}
  \caption{Learnable $\alpha$ trajectory and comparison with baseline. \textbf{Left:} $\alpha$ decreases 71\% from initialization, converging to $\sim$0.03. \textbf{Center:} MTP accuracy comparison---baseline (no context injection) achieves 57.9\% versus the regularized variant's 54.6\%. \textbf{Right:} NTP-MTP gap comparison---baseline achieves a lower gap (24.6\%) than the regularized variant (27.4\%). The model actively learns to minimize context injection because it hurts performance.}
  \label{fig:alpha_trajectory}
  \end{figure}

\section{2-Layer vs 4-Layer P-EAGLE Comparison}
\label{appendix:2layer}

Table~\ref{tab:2layer_comparison} compares acceptance length between 2-layer and 4-layer P-EAGLE configurations. The 2-layer variant was trained for fewer epochs and serves as a capacity-latency tradeoff point: while it offers lower per-forward-pass latency, 4-layer P-EAGLE consistently achieves higher acceptance lengths across all benchmarks.

\begin{table}[h]
\centering
\caption{Acceptance length comparison between 2-layer and 4-layer P-EAGLE. Speculation depth $K=5$, max new tokens 2048. Percentages denote change relative to AR EAGLE-3.}
\label{tab:2layer_comparison}
\resizebox{\columnwidth}{!}{
\begin{tabular}{llccc}
\toprule
\textbf{Model} & \textbf{Dataset} & \textbf{AR EAGLE-3} & \textbf{P-EAGLE (2L)} & \textbf{P-EAGLE (4L)} \\
\midrule
\multirow{4}{*}{GPT-OSS 120B}
 & HumanEval & 3.5 & 3.2 ($-$6.4\%) & \textbf{3.5} (0.0\%) \\
 & MT-Bench  & 2.7 & 2.7 (+0.8\%) & \textbf{2.9} (+10.2\%) \\
 & GSM-8K    & 3.3 & 3.2 ($-$2.7\%) & \textbf{3.5} (+5.2\%) \\
 & \textit{Average} & \textit{3.1} & \textit{3.0 ($-$3.2\%)} & \textit{\textbf{3.3} (+4.5\%)} \\
\midrule
\multirow{4}{*}{GPT-OSS 20B}
 & HumanEval & 3.7 & 3.3 ($-$11.1\%) & \textbf{3.8} (+2.4\%) \\
 & MT-Bench  & 3.4 & 2.9 ($-$13.8\%) & \textbf{3.4} (+1.5\%) \\
 & GSM-8K    & 3.9 & 3.5 ($-$10.3\%) & \textbf{4.0} (+3.1\%) \\
 & \textit{Average} & \textit{3.7} & \textit{3.2 ($-$12.4\%)} & \textit{\textbf{3.7} (+2.5\%)} \\
\midrule
\multirow{4}{*}{Qwen3-Coder 30B}
 & HumanEval & 4.4 & 4.2 ($-$5.0\%) & \textbf{4.5} (+3.7\%) \\
 & MT-Bench  & 3.0 & 2.8 ($-$6.3\%) & \textbf{3.0} (+0.3\%) \\
 & GSM-8K    & 3.1 & 2.8 ($-$10\%) & \textbf{3.2} (+1.0\%) \\
 & \textit{Average} & \textit{3.5} & \textit{3.3 ($-$6.9\%)} & \textit{\textbf{3.6} (+2.0\%)} \\
\bottomrule
\end{tabular}
}
\end{table}

Key observations: (1) 2-layer P-EAGLE achieves 93--97\% of the baseline acceptance length on average, while 4-layer matches or exceeds it. (2) The gap is most pronounced on GPT-OSS 20B, where 2-layer shows $-$12.4\% average degradation while 4-layer achieves +2.5\% improvement. (3) For deployment scenarios prioritizing lower drafting latency over acceptance length, 2-layer P-EAGLE remains a viable option.




\end{document}